%% file: main.tex
\documentclass[11pt]{article}

\usepackage[final]{acl}

\usepackage{natbib}
\usepackage{times}
\usepackage{latexsym}

\usepackage[T1]{fontenc}

\usepackage[utf8]{inputenc}

\usepackage{microtype}
\usepackage{inconsolata}
\usepackage{macro}
\usepackage{graphicx}
\usepackage{amsmath}
\usepackage{times}
\usepackage{latexsym}
\usepackage{makecell}
\usepackage{balance}
\usepackage[T1]{fontenc}

\usepackage[utf8]{inputenc}

\usepackage{microtype}

\usepackage{inconsolata}

\usepackage{graphicx}
\usepackage{balance}
\usepackage{caption}
\usepackage{subcaption}
\usepackage{booktabs}
\usepackage{algorithm,algorithmicx,algpseudocode}
\usepackage{enumitem}
\usepackage{multirow}

\makeatletter
\def\thanksnosymbol#1{\protected@xdef\@thanks{\@thanks
        \protect\footnotetext{#1}}}
\makeatother
\usepackage{hyperref}

\title{Long-Context Demonstration Selection Using State Space Models}

\author{Ziniu Zhang \\
  Northeastern University \\
  \texttt{zhang.zini@northeastern.edu} \\\And
  Zhenshuo Zhang \\
  Northeastern University \\
  \texttt{zhang.zhens@northeastern.edu} \\\AND
  Ruoxuan Xiong \\
  Emory University \\
  \texttt{ruoxuan.xiong@emory.edu} \\\And
  Gene Cooperman \\
  Northeastern University\\
  \texttt{gene@ccs.neu.edu} \\\And
  Hongyang R. Zhang \\
  Northeastern University \\
  \texttt{ho.zhang@northeastern.edu}}

\begin{document}
\maketitle

\input{intro}

\input{content}

\bibliography{./bibliography/ref}
\clearpage

\appendix
\input{appendix}

\end{document}

%% file: intro.tex
\begin{abstract}
    We study the problem of demonstration selection, which involves selecting a subset of examples for prepending to a query to a language model. This problem is closely related to in-context learning and language model inference. Since the inference cost of a transformer model scales quadratically with sequence length, the selection problem becomes especially challenging in a long-context scenario. In this paper, we tackle this problem by building on state space models (SSMs), which require only linear inference time given the input. Our approach involves two algorithms. The first learns a small set of SSMs through \emph{distillation} of a (trained) transformer model. We partition all the layers into consecutive groups. Then for each group, we estimate a separate state space model to replicate the input-output behavior within the adjacent layers. Second, we map the distilled model outputs to a small set of tokens, and apply these embeddings for demonstration selection in downstream applications. We perform extensive experiments in both synthetic and real-world datasets to validate our approach. We demonstrate that the distilled SSMs only incur an approximation error of less than $0.7\%$ relative to the true output. In downstream evaluation, we show that on several text classification and reasoning tasks, our approach reduces FLOPs by $14.2\times$ and improves accuracy by $6.48\%$ relative to baseline demonstration selection methods.
\end{abstract}

\section{Introduction}

Language models increasingly rely on answering to long-context prompts at inference time~\citep{oncescu2025flash}. One approach is to condition the model on demonstrations drawn from historical trajectories \cite{xiong2025automated}. Specifically, given a query, the goal is to select a subset of $k$ demonstrations from a large candidate pool of size $n$. This problem is known as \textit{demonstration selection}, which is closely related to in-context learning \cite{garg2022can,zhang2025linear}.

\begin{figure*}[t]
    \centering
    \includegraphics[width=0.996\textwidth]{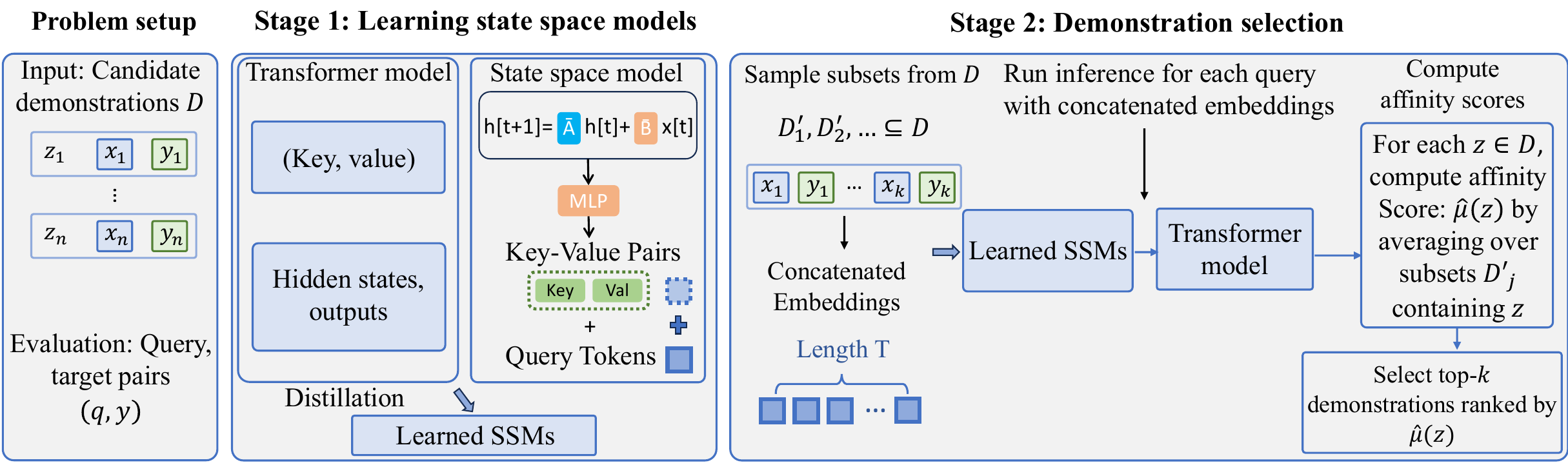}
    \caption{Overview of our approach: Given a set of candidate demonstrations and query-target pairs from an unknown distribution, our approach learns a set of state space models through distillation from a pretrained model. Then, we use the extracted embeddings from the state space models for demonstration selection. We sample multiple subsets of demonstrations and evaluate the concatenated subset embeddings. We aggregate the subset losses to an affinity score for every demonstration and select the top-$k$ demonstrations ranked by the scores.}
    \label{fig_lcc_workflow}
\end{figure*}

The challenge in selecting demonstrations is that the effect of one demonstration often depends on which other demonstrations appear with it in the prompt. Evaluating candidate subsets often requires repeatedly running the language model on many different prompts~\citep{zhang2025linear}. For transformer models, the cost of processing long prompts grows {quadratically} with the total context length $T$, where $T$ scales with both the number and length of demonstrations. This creates a computational bottleneck for long-context demonstration selection, since exhaustively evaluating many candidate subsets is infeasible as $k$ increases.

This challenge connects demonstration selection to a broader problem of efficient long-context inference. Existing approaches largely fall into three categories. First, state space models (SSMs)~\citep{gu2020hippo,gu2022efficiently,gu2024mamba} and related hybrid architectures~\citep{ren2025samba,oncescu2025flash} achieve efficient sequence modeling, but typically require retraining the underlying language model. Second, training-free methods accelerate inference by restricting attention to local windows or selected tokens~\citep{xiao2024efficient,xiao2025efficient}. Third, other related methods~\citep{mu2023learning,chevalier2023adapting,ge2024context} require processing the entire prompt.

In this paper, we build on state space models to scale up demonstration selection to long-contexts at inference time. The key idea is to train a group of SSMs to map each candidate subset's concatenated demonstrations into a small set of tokens. These tokens are computed once per subset and reused across all queries evaluated with that subset. This substantially reduces the cost of evaluating many candidate demonstration subsets. Our algorithm consists of three main components. 

First, we learn a small number of state space models through distillation. We partition the transformer layers into disjoint groups of consecutive layers and assign one SSM to each group. Each SSM scans the embeddings of the long demonstration prompt into a compact hidden representation, which is then mapped into a small set of key-value states. The output key-value states are then combined together as the tokens that transformers can process directly. These tokens serve as a substitute for the original long prompt during inference.

Second, we design a demonstration selection algorithm based on subset sampling \cite{li2023identification,li2023boosting,li2024scalable,li2024gradex,zhang2025linear}. Crucially, this algorithm now runs on top of the (distilled) state space models. We repeatedly sample candidate subsets and evaluate their predictive performance using the representation produced by our approach. We then estimate the contribution of each demonstration based on how subsets containing that demonstration perform.

Finally, we conduct empirical analyses of the proposed framework. We find that our approach preserves the predictive distribution of the original model with relative error below $0.6\%$. Across classification and reasoning benchmarks, our approach reduces FLOPs by up to $14.2\times$ and improves downstream accuracy by $6.48\%$ over baselines.

In summary, we design an algorithm to scale up demonstration selection to long-context inference. First, we build SSMs that map long demonstration prompts into a small set of reusable tokens. Second, we develop an affinity-based subset selection algorithm that leverages this representation to evaluate many candidate demonstration subsets. Third, we validate the proposed algorithms on synthetic, classification, and reasoning benchmarks. The code to reproduce our experiments is available at \href{https://github.com/VirtuosoResearch/Long-context-demonstration-selection}{\url{https://github.com/VirtuosoResearch/Long-context-demonstration-selection}}.

%% file: content.tex
\section{Preliminaries}\label{sec_setup}

We study demonstration selection in constructing a prompt. Let $D$ denote the set of candidate demonstrations of size $n$. Let $(q, y)\sim\mathcal{P}$ denote a query-target pair drawn from an unknown distribution. We select a subset $D' \subseteq D$ containing $k$ demonstrations with a fixed ordering for every query. Given a model $f_W$ and a loss function $\ell: \cX \times \cY \rightarrow \real$ supported on an input domain $\cX$ and output domain $\cY$, the \emph{demonstration (subset) selection problem} is defined as the following minimization:
\begin{equation*}
    \min_{\substack{D' \subseteq D:\, \abs{D'} = k}} \mathbb{E}_{(q,y)\sim\mathcal{P}} \left[ \ell\!\left(f_W(D',q),y\right) \right].
\end{equation*}

For the selected demonstrations $D'$, we concatenated them as the input demonstration prompt. Let $X = [x_1, \dots, x_T]^\top \in \mathbb{R}^{T \times H}$
denote the token embeddings of the prompt, where $T$ is the total number of tokens and $H$ is the embedding dimension. Since $T$ increases with both the number and length of demonstrations, long-context prompting can produce very large input sequences. In a transformer model, the computational cost of processing the prompt grows quadratically with $T$ because self-attention computes pairwise interactions between all tokens. As a result, evaluating the inference outcome for a large $T$ is computationally expensive.

A natural approach for processing long prompts is to use state space models (SSMs), which summarize the input sequence through recurrent updates. Given the prompt embeddings $X$, a discrete SSM updates its state at each token position $t$ as
\begin{equation}
    h_t=\bar{A}h_{t-1}+\bar{B}x_t,
\end{equation}
where $h_t\in\mathbb{R}^{N}$ is the $N$-dimensional state, $\bar{A}\in\mathbb{R}^{N\times N}$ is the state-transition matrix, and $\bar{B}\in\mathbb{R}^{N\times H}$ maps the token embedding $x_t$ into the state space. The SSM processes each token once and has linear complexity in the sequence length $T$ \cite{gu2022efficiently}. 

Previous work has shown that the HiPPO-LegS matrix provides a structured construction of the state-transition matrix by maintaining an online polynomial projection of the input history~\citep{gu2020hippo,gu2024mamba}. The state approximates the coefficients of past inputs under a scaled Legendre basis. Lower-order coefficients capture the main structure of the sequence history, and higher-order coefficients represent finer variations. Thus, a finite-dimensional HiPPO state provides a compact summary of a long input sequence.

In the in-context learning setting, we hypothesize that task-relevant information in a large demonstration set can be preserved by low-order Legendre basis components of the HiPPO matrix. The SSM scans the complete input once. Thus, it can replace a transformer model's inference for the demonstrations and reduce the inference complexity from ${O}(T^2)$ to ${O}(T)$.

Motivated by these existing results, a natural question is whether one could adapt SSMs to tackle the demonstration subset selection problem.
\begin{itemize}
    \item First, how can we design the SSM for the in-context learning setting?
    \item Second, how can we utilize the SSM for demonstration selection?
\end{itemize}
In the next section, we design algorithms to answer each of the above questions, respectively.

\section{Our Approach}\label{sec_lcc}

We build SSMs for efficient long-context inference. First, we present the architecture of the SSMs with theoretical and empirical evidence supporting our hypothesis. We also introduce a distillation procedure for training the SSMs. Then, we propose a demonstration selection method on the embeddings and use controlled synthetic experiments to evaluate its affinity estimates.

\subsection{Learning the State Space Models}\label{sec_glcc}

We design SSMs with the HiPPO-LegS matrix to map long demonstration prompts while preserving the dominant low-order structure of the sequence. 
In a transformer, the prompt is first processed to construct a layer-wise key-value (KV) cache, where each transformer layer stores key and value representations for all prompt tokens. These cached representations are later used during decoding.

Instead of computing the exact KV cache for the full prompt, we use SSMs to process it. The SSM sequentially scans the prompt and summarizes the long sequence into a compact hidden representation. We partition the $m$ layers of the transformer model $f$ into $g$ disjoint groups of consecutive layers to yield a lower-dimensional target for each group. Each group is assigned an independent SSM that predicts the KV states only for its corresponding group of layers.

Specifically, for each layer group $G_i$ for $i = 1,\dots,g$, the SSM processes the prompt embeddings $X$ and produces a final hidden state $h_T^{(i)} \in \mathbb{R}^N$. A multilayer perceptron $f_W^{(i)}$ then projects $h_T^{(i)}$ into a separate pair of key-value states for every transformer layer in the group:
\begin{equation}
    \left\{
        \bigl(\tilde{K}^{(l)},\tilde{V}^{(l)}\bigr)
    \right\}_{l\in G_i}
    =
    f_W^{(i)}(h_T^{(i)}(X)). \notag
\end{equation}
For each transformer layer $l$, the original KV states have shape
${T\times n_{\mathrm{kv}}\times d}$, while the projected states have shape ${n_v\times n_{\mathrm{kv}}\times d}$, where $n_{\mathrm{kv}}$ is the number of key-value heads and $d$ is the dimension of each head. The KV states generated from the outputs of all $g$ SSMs are aligned at the same $n_v$ token positions and assembled across layer groups to form the layer-wise KV cache of these $n_v$ tokens. Thus, the projection preserves the key-value head and head-dimension axes while mapping $T$ tokens to $n_v$ tokens for a transformer to process directly.

To analyze the proposed architecture, we first run the transformer on the complete input prompt using a standard full-attention forward pass and record the resulting key-value states at every layer. We represent the exact transformer KV cache of each layer group $G_i$ by a matrix $C^{(i)}$. This matrix stacks the key and value states from all layers and attention heads in the group, and retains the $T$ token positions as its rows. 

\begin{proposition}
\label{prop_compression}
Let $\{\phi_r\}_{r=1}^{T}$ be an orthonormal discrete Legendre basis along the token dimension, so that $\phi_r^\top\phi_{r'}=\mathbbm{1}_{r=r'}$. For each $G_i$, let $C^{(i)}\in\mathbb{R}^{T\times D_i}$, where $D_i$ is the number of stacked cache features, and $c_r^{(i)}:=\bigl(C^{(i)}\bigr)^\top\phi_r\in\mathbb{R}^{D_i}, r=1,\ldots,T$.
The orthogonal decomposition of $C^{(i)}$ is $\sum_{r=1}^{T}\phi_r\bigl(c_r^{(i)}\bigr)^\top$.
For $N<T$, define the orthogonal projection matrix of rank $N$ onto $\operatorname{span}\{\phi_1,\ldots,\phi_N\}$ by
$\Pi_N:=\sum_{r=1}^{N}\phi_r\phi_r^\top\in\mathbb{R}^{T\times T}.$
It follows that $\Pi_N C^{(i)}=\sum_{r=1}^{N}\phi_r\bigl(c_r^{(i)}\bigr)^\top$ and $C^{(i)}-\Pi_N C^{(i)}=\sum_{r=N+1}^{T}\phi_r\bigl(c_r^{(i)}\bigr)^\top$.
Therefore, for any $s>0$,
\begin{align*}
    \left\|C^{(i)}-\Pi_N C^{(i)}\right\|_F^2
    \le \sum_{r=N+1}^{T} \frac{r^{2s}}{(N+1)^{2s}} \left\|c_r^{(i)}\right\|_2^2.
\end{align*}
\end{proposition}

This shows that the error stems from Legendre components omitted by the first $N$ modes. When the coefficients are concentrated at low frequencies, this tail is small. Increasing the SSM state dimension $N$ retains more components and further reduces the approximation error. The proof is provided in Appendix~\ref{app_group_error}.

\begin{algorithm}[t!]
\caption{Learning SSMs via Key, Value, and Output Distillation}
\label{alg_lcc}
\textbf{Input}: Embedding and query pairs $\{X,q\}$, parameter initialization $\Theta=\{\bar B^{(i)}, W^{(i)}\}_{i=1}^g$ \\
\textbf{Require}: Groups $\{G_i\}_{i=1}^{g}$, $g$ SSMs each with a fixed $\bar A$ and variable weight matrix $\bar B^{(i)}$ and MLPs $\{f_W^{(i)}\}_{i=1}^g$, an $m$-layer transformer, $n_v$, parameters $\lambda_1,\lambda_2$, and learning rate $\eta$ \\
\textbf{Output}: Trained parameters
$\hat \Theta$
\begin{algorithmic}[1]
\For{each $(X,q)$}
    \State Extract $\{K^{(l)},V^{(l)}\}_{l=1}^{m}$ from transformer
    \State $T\leftarrow$ Length of $X$
    \For{$i\in \{1,\ldots,g\}$}
        \State
        $\{(\tilde K^{(l)},\tilde V^{(l)})\}_{l\in G_i}
        \leftarrow f_W^{(i)}(h_T^{(i)}(X))$
    \EndFor
    \For{$i \in \set{1,...,n_v}$, $l \in \set{1, \dots, m}$}
        \State $\kappa_1^{(i,l)}\leftarrow 1-\cos(\tilde{K}^{(l)}_i,\operatorname{Pool}(K^{(l)})_i)$
        \State $\kappa_2^{(i,l)} \leftarrow 1-\cos(\tilde{V}^{(l)}_i,\operatorname{Pool}(V^{(l)})_i)$
    \EndFor
    \State $\Theta\leftarrow
    \Theta-\eta\nabla_\Theta \big(\frac{1}{2m}\sum_{l=1}^{m}\sum_{i=1}^{n_v} ( \kappa_1^{(i,l)} + \kappa_2^{(i,l)} )\big)$ \hfill \textcolor{gray}{\textit{// Key-value distillation}}
\EndFor
\For{each $(X,q)$}
    \For{$i\in \{1,\ldots,g\}$}
        \State
        $\{(\tilde K^{(l)},\tilde V^{(l)})\}_{l\in G_i}
        \leftarrow f_W^{(i)}(h_T^{(i)}(X))$
    \EndFor
    \State $(P_1,\{\psi_1^{(l)}\})\leftarrow f((\tilde K,\tilde V),q)$
    \State $(P_2,\{\psi_2^{(l)}\})\leftarrow f(X,q)$
    \State $L \leftarrow
    \lambda_1D_{\mathrm{KL}}(P_1\|P_2)
    +
    \frac{\lambda_2}{m}\sum_{l=1}^{m}
    ( 1-\cos(\psi_1^{(l)},\psi_2^{(l)}) )$
    \State $\Theta\leftarrow
    \Theta-\eta\nabla_\Theta L$ \hfill \textcolor{gray}{\textit{// Output distillation}}
\EndFor
\State \textbf{return} $\hat \Theta$
\end{algorithmic}
\end{algorithm}

We empirically verify this on Qwen-7B-Instruct and Llama-3-8B-Instruct models with $G{=}4$ layer groups and $k{=}50$ demonstrations. The rapid coefficient decay in Figure~\ref{fig_legendre_spectrum} indicates that a finite-dimensional HiPPO state can retain most of the information in the KV cache. Based on this observation, the SSMs process the entire $T$-token input using recurrent scans with $O(T)$ cost for a fixed number of groups. The transformer then processes the resulting $n_v$ tokens with a cost of $O(n_v^2)$. Since $n_v$ is fixed and $n_v\ll T$, the total prefix-processing complexity is $O(T+n_v^2)=O(T)$. In the next part, we introduce a two-stage distillation method to train the SSMs to produce KV states that preserve the original transformer function output.

\begin{figure}[t]
\centering
\begin{subfigure}{0.485\linewidth}
\centering
\includegraphics[width=\linewidth]{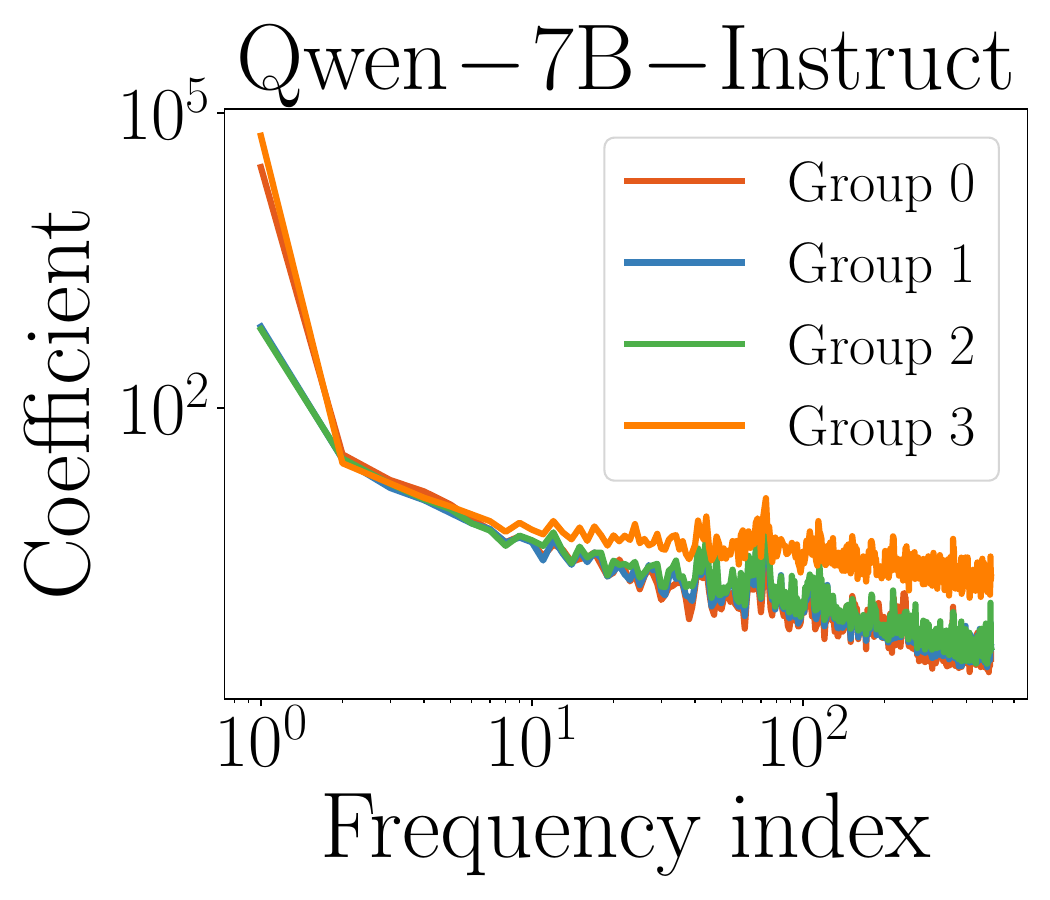}
\end{subfigure}
\hfill
\begin{subfigure}{0.485\linewidth}
\centering
\includegraphics[width=\linewidth]{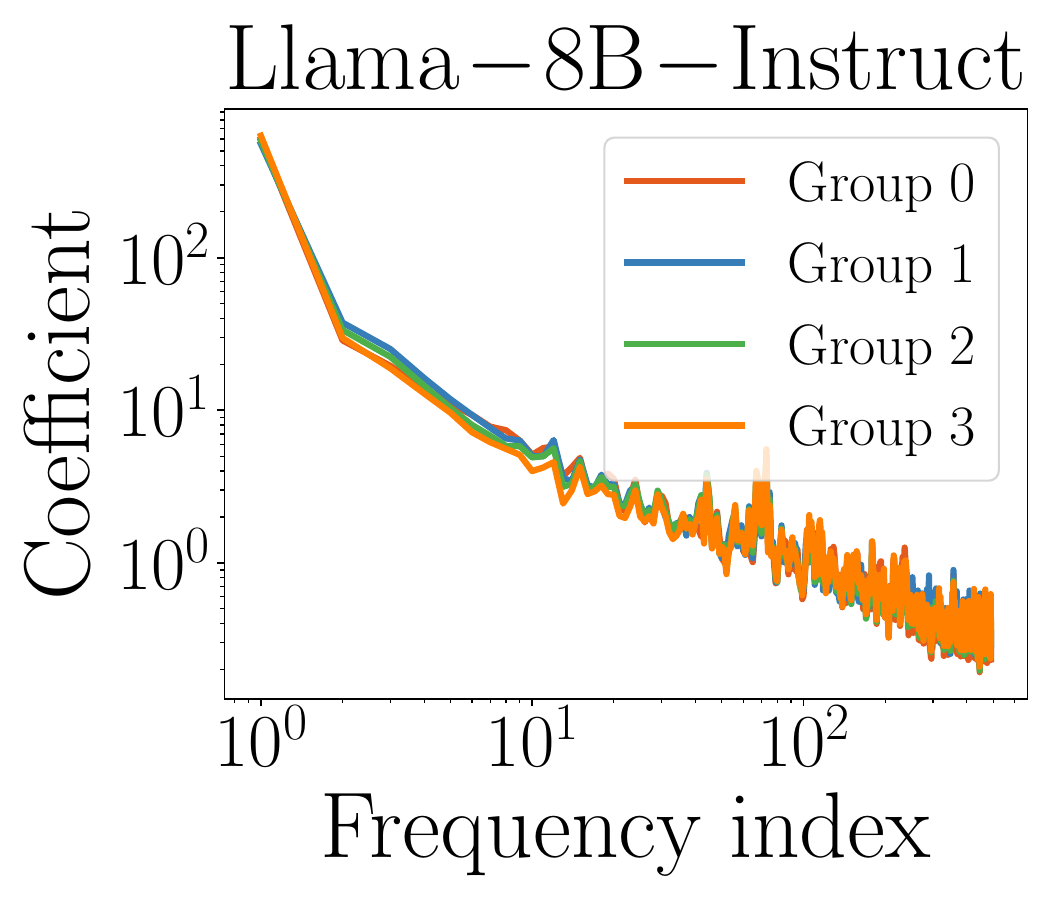}
\end{subfigure}
\caption{Mean squared discrete Legendre coefficients of KV caches on Qwen-7B-Instruct and Llama-3-8B-Instruct with $G{=}4$ layer groups and $k{=}50$ demonstrations. The coefficients decay rapidly across groups, and the KV cache is dominated by low-order HiPPO-LegS components.}
\label{fig_legendre_spectrum}
\end{figure}

\paragraph{Key, value, and output distillation.}%
To train the SSMs, we design a distillation procedure. First, we align the generated tokens with a valid KV space. Given the $m$-layer model $f_W$, we extract its exact KV cache $\{K^{(l)}, V^{(l)}\}_{l=1}^{m}$. To resolve the length mismatch, we apply pooling to the $T$ states, producing $n_v$ pooled targets $\operatorname{Pool}(K)$ and $\operatorname{Pool}(V)$. 
For the $i$-th layer group, we keep the HiPPO-LegS transition matrix fixed and train only the SSM input projection $\bar B^{(i)}$ and the parameters of the KV projection MLP $f^{(i)}$. We denote all trainable parameters by $\Theta = \left\{ \bar B^{(i)},W^{(i)} \right\}_{i=1}^{g}$.
The objective is:
\begin{align}
&\frac{1}{2m}\sum_{l=1}^{m}\sum_{i=1}^{n_v} \Big( \left(1-\cos\!\left( \tilde{K}^{(l)}_i,\operatorname{Pool}(K^{(l)})_i \right)\right) \notag\\
&\qquad
+ \left(1-\cos\!\left( \tilde{V}^{(l)}_i,\operatorname{Pool}(V^{(l)})_i \right)\right) \Big). \notag
\end{align}

\begin{table*}[t]
\centering
{\small
\begin{tabular}{@{}l|l|cccc@{}}
\toprule
Model & \# Demonstrations & SST-2 & Poem Sentiment & Coin Flip & Modular Addition \\ \midrule
\multirow{2}{*}{Qwen-1.5B} & $k=50$ & $2.5_{\pm0.1}\times10^{-3}$ & $1.3_{\pm0.1}\times10^{-3}$ & $1.9_{\pm0.5}\times10^{-3}$ & $1.3_{\pm0.1}\times10^{-3}$ \\
& $k=75$ & $2.8_{\pm0.2}\times10^{-3}$ & $1.9_{\pm0.2}\times10^{-3}$ & $2.1_{\pm0.2}\times10^{-3}$ & $1.8_{\pm0.4}\times10^{-3}$ \\\midrule
\multirow{2}{*}{Qwen-3B} & $k=50$ & $6.1_{\pm0.2}\times10^{-3}$ & $2.3_{\pm0.3}\times10^{-3}$ & $2.1_{\pm0.8}\times10^{-3}$ & $2.5_{\pm0.3}\times10^{-3}$ \\
& $k=75$ & $6.3_{\pm0.2}\times10^{-3}$ & $2.9_{\pm0.2}\times10^{-3}$ & $4.1_{\pm1.1}\times10^{-3}$ & $2.1_{\pm0.7}\times10^{-3}$ \\\midrule
\multirow{2}{*}{Llama-8B} & $k=50$ & $2.3_{\pm0.1}\times10^{-3}$ & $2.7_{\pm0.1}\times10^{-3}$ & $1.5_{\pm0.5}\times10^{-3}$ & $3.8_{\pm1.0}\times10^{-3}$ \\
& $k=75$ & $2.5_{\pm0.2}\times10^{-3}$ & $2.9_{\pm0.2}\times10^{-3}$ & $1.9_{\pm0.2}\times10^{-3}$ & $4.8_{\pm0.9}\times10^{-3}$ \\
\bottomrule
\end{tabular}}
\caption{Relative error of the predictive output between the full demonstration prefix $p$ and the tokens $\hat p$. We vary the number of demonstrations on multiple tasks. Lower values indicate higher fidelity to the original model output. We run three times to compute the mean and standard deviations.}\label{tab_error}
\end{table*}

Next, we train the SSM using joint supervision of outputs and intermediate states to preserve predictive accuracy and positional structure. For a query $q$, let $\Pr[y \mid \tilde{K}, \tilde{V}, q]$ and $\Pr[y \mid X, q]$ denote the output distributions of the LLM conditioned on the SSMs' output and original demonstrations, respectively. We denote these two distributions by $P_1$ and $P_2$, respectively. We additionally align the corresponding query hidden states $\psi_1^{(l)}$ and $\psi_2^{(l)}$ across all layers. The entire loss objective, denoted by $\hat L (\Theta)$, is: 

\begin{align}
\lambda_1 D_{\mathrm{KL}}\!\left(P_1\|P_2\right) \notag
+ \frac{\lambda_2}{m} \sum_{l=1}^{m} \left( 1-\cos\!\left(\psi_1^{(l)},\psi_2^{(l)}\right) \right).\notag
\end{align}
where $\lambda_1$ and $\lambda_2$ adjust the relative weight of the respective terms.
The entire procedure is summarized in Algorithm \ref{alg_lcc}.

To validate the effectiveness of the above distillation procedure, we evaluate Algorithm \ref{alg_lcc} on multiple tasks using $k$ demonstrations from $50$ to $75$. As shown in Table~\ref{tab_error}, the relative error between the logits from the full prefix $p$ and tokens $\hat{p}$ remains strictly below $0.7\%$ across all settings. This confirms that Algorithm~\ref{alg_lcc} robustly preserves the original predictive behavior. We defer the omitted result where $k=25$ to Table~\ref{tab_error_k=25} in Appendix \ref{app_omitted_results}.

\subsection{Demonstration Selection Using SSMs}\label{sec_aslcc}

We now utilize Algorithm \ref{alg_lcc} for demonstration selection. We first introduce a randomized estimator that quantifies the marginal utility of each candidate via subset sampling. We define the affinity score, $\mu(z)$, of a demonstration example $z\in\mathcal{D}$ as
\[ - \mathbb{E}_{\substack{ D'\subseteq{D}, z\in D'}} [\mathbb{E}_{(q,y)\sim\mathcal{P}} [\ell(f_W(D',q),y)]]. \]
Thus, $\mu(z)$ is the negative expected inference loss averaged over subsets of size $k$ that contain $z$.

\begin{algorithm}[t!]
\caption{Demonstration Selection Using SSMs}
\label{alg_aslcc}
\textbf{Input}: Candidate set $D$, validation set $\mathcal V$, subset size $k$, number of sampled subsets $M$, parameter initialization $\Theta$ \\
\textbf{Require}: Transformer model $f$, groups $\{G_i\}_{i=1}^{g}$, HiPPO-LegS matrix $\bar A$ \\
\textbf{Output}: $k$ selected demonstrations
\begin{algorithmic}[1]
\State $\hat\Theta \leftarrow$ Algorithm~\ref{alg_lcc}
\For{$j \in \{1,\ldots,M\}$}
    \State Sample $D_j'\subseteq D$ of size $k$ 
    \State $X\leftarrow$ Embeddings of $D_j'$
    \State $T\leftarrow$ Length of $X$
    \For{$i\in \{1,\ldots,g\}$}
        \State
        $\{(\tilde K^{(l)},\tilde V^{(l)})\}_{l\in G_i}
        \leftarrow f_W^{(i)}(h_T(X))$
    \EndFor
    \State $\tilde C\leftarrow
    \{(\tilde K^{(l)},\tilde V^{(l)})\}_{l=1}^{m}$
    \State $\displaystyle
    \hat L_j\leftarrow
    \frac{1}{|\mathcal V|}
    \sum_{(q,y)\in\mathcal V}
    \ell\!\left(f(\tilde C,q),y\right)$
\EndFor
\For{$z\in D$}
    \State $\displaystyle
    \hat\mu(z)\leftarrow-
    \frac{\sum_{j=1}^{M}
    \mathbf 1_{\{z\in D'_j\}}\hat L_j}
    {\sum_{j=1}^{M}
    \mathbf 1_{\{z\in D'_j\}}}$
\EndFor
\State \textbf{return} Top-$k$ demonstrations according to $\hat\mu(z)$ for all $z \in D$
\end{algorithmic}
\end{algorithm}

To estimate $\mu(z)$, we uniformly sample $M$ subsets $\{D'_1,\dots,D'_M\}$ from $D$ sized $k$. Then, we estimate the affinity, denoted by $\hat{\mu}(z)$, as
\[ -\frac{ \sum_{i=1}^{M} \mathbf{1}_{z\in D'_i} \sum_{(q,y)\in \mathcal{V}} \ell(f_W(D'_i,q),y)}{ \sum_{i=1}^{M} \mathbf{1}_{z\in D'_i}}. \]
It is possible to show that with $M=O(n/(k\epsilon^2))$ sampled subsets, one can reduce the uniform estimation error down to $\epsilon$ \cite{li2023identification,li2023boosting}. Afterwards, we rank the candidates by $\hat{\mu}(z)$ and select the $k$ demonstrations with the highest scores. Taken together, we summarize the entire procedure in Algorithm \ref{alg_aslcc}.

\paragraph{A case study of linear functions.}
To evaluate our algorithm, we adopt the synthetic linear regression setting from \citet{garg2022can} using a two-layer transformer. Both demonstrations and queries are generated via $y_i = \langle \beta^{(c)}, x_i \rangle + \epsilon_i$, where each input $x_i\in\mathbb{R}^{20}$ is independently sampled from the standard Gaussian distribution $\mathcal{N}(0,\id_{20})$. To simulate multiple queries clustering around shared underlying tasks, each coefficient vector $\beta^{(c)}$ is drawn randomly from a set of mutually orthogonal anchors, ensuring strict task decoupling. Finally, we inject independent Gaussian noise $\epsilon_i \sim \mathcal{N}(0, \sigma^2)$ to prevent the model from trivially converging to an exact solution after observing $20$ examples.

We evaluate the predictive performance of our proposed algorithm against two baseline strategies: random demonstration selection (random-$k$) and embedding similarity-based selection (top-$k$). As shown in Figure~\ref{fig_lr_main}, our method achieves the lowest estimation error in the linear setting. This predictive advantage robustly extends to the non-linear setting, which is a two-layer ReLU network (Figure~\ref{fig_relu_main}). Note that the model is trained with $40$ data points as demonstrations, thus the error increases after the number of samples reaches $40$. 

By leveraging the SSMs during subset evaluation, our approach reduces FLOP overhead by $15.7\times$ compared to naive inference (Figure~\ref{fig_lr}), ensuring high scalability for massive contexts.

\begin{figure}[t!]
\centering
\begin{subfigure}{0.49\linewidth}
\centering
\includegraphics[width=0.98\linewidth]{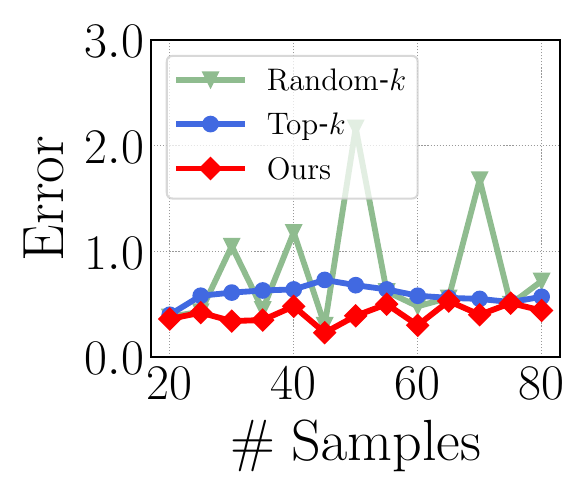}
\caption{Linear Functions}
\label{fig_lr_main}
\end{subfigure}
\hfill
\begin{subfigure}{0.49\linewidth}
\centering
\includegraphics[width=0.98\linewidth]{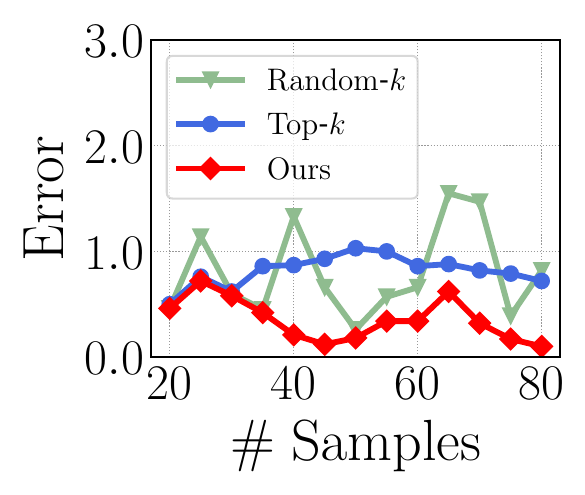}
\caption{Nonlinear Functions}
\label{fig_relu_main}
\end{subfigure}
\begin{subfigure}{0.49\linewidth}
\centering
\includegraphics[width=0.98\linewidth]{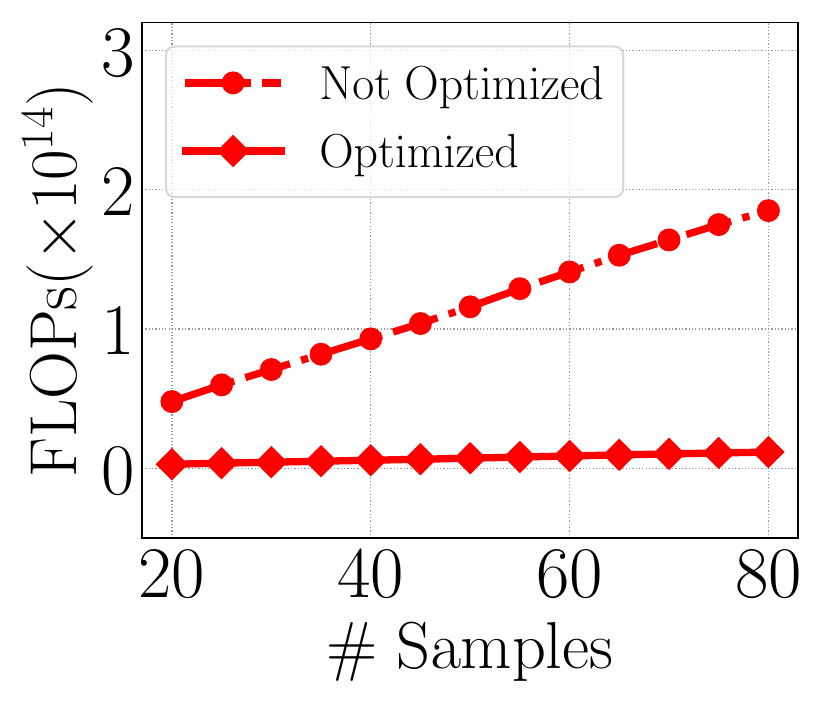}
\caption{Linear Functions}
\end{subfigure}
\begin{subfigure}{0.49\linewidth}
\centering
\includegraphics[width=0.98\linewidth]{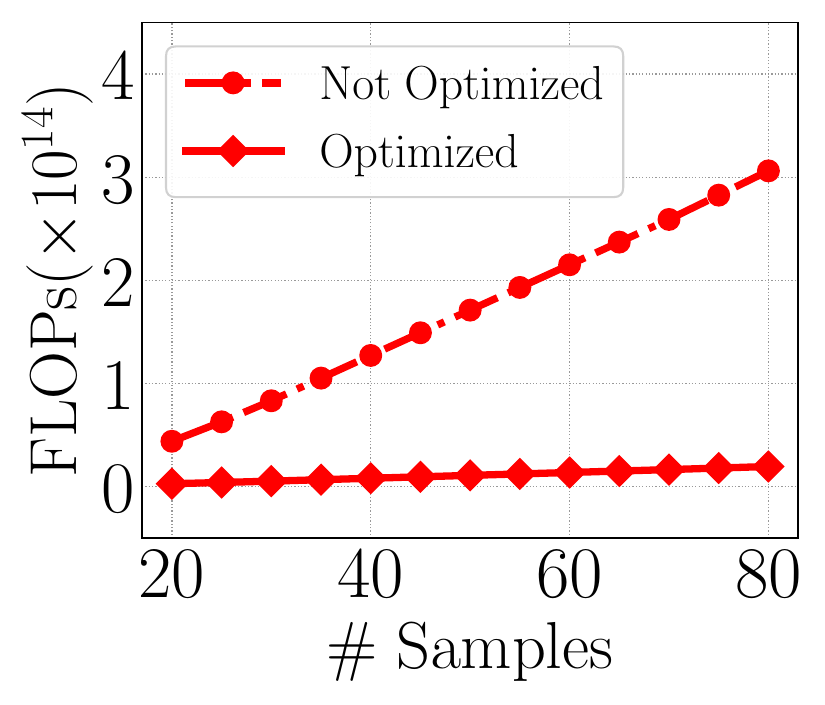}
\caption{Nonlinear Functions}
\end{subfigure}
\caption{Evaluation of our approach for linear and nonlinear functions. The top panel provides a comparison with random-$k$ and top-$k$ selection methods. The lower panel illustrates the computational costs relative to full inference (denoted as not optimized).}
\label{fig_lr}
\end{figure}

\begin{table*}[t]
\centering
\resizebox{0.995\textwidth}{!}
{\small
\begin{tabular}{@{}l|c|ccc|ccc|ccc@{}}
\toprule
\multirow{2}{*}{Method} & Window & \multicolumn{3}{c|}{MMLU} & \multicolumn{3}{c|}{Modular Addition} & \multicolumn{3}{c}{GSM8K} \\ 
& Length & Error & FLOPs & Memory & Error & FLOPs & Memory & Error & FLOPs & Memory \\
\midrule
Dense & - & - & $6.53e^{15}$ & $17.31$G & - & $1.05e^{15}$ & $13.65$G & - & $2.59e^{15}$ & $15.75$G \\
StreamingLLM & $64$ & $1.1\%$ & $7.75e^{13}$ & $15.90$G & $9.2\%$ & $1.90e^{13}$ & $12.95$G & $5.6\%$ & $1.26e^{14}$ & $14.06$G \\
LM-Infinite & $64$ & $1.2\%$ & $7.75e^{13}$ & $15.91$G & $10.9\%$ & $1.91e^{13}$ & $12.97$G & $5.1\%$ & $1.27e^{14}$ & $14.06$G \\
Gist Tokens & $64$ & $7.2\%$ & $1.66e^{15}$ & $14.45$G & $2.9\%$ & $7.50e^{14}$ & $13.72$G & $40.4\%$ & $2.37e^{14}$ & $14.43$G \\
ICAE & $64$ & $16.5\%$ & $3.79e^{14}$ & $17.02$G & $1.4\%$ & $1.36e^{14}$ & $14.27$G & $0.5\%$ & $3.33e^{14}$ & $14.87$G \\
DuoAttention & $64$ & $2.7\%$ & $8.03e^{13}$ & $16.24$G & $7.7\%$ & $1.93e^{13}$ & $13.25$G & $4.9\%$ & $1.31e^{14}$ & $14.25$G \\
I2CL & - & $14.3\%$ & $2.62e^{14}$ & $15.89$G & $0.7\%$ & $8.29e^{14}$ & $13.14$G & $11.7\%$ & $8.56e^{14}$ & $13.86$G \\
BSA & $\sim400$ & $3.1\%$ & $7.86e^{13}$ & $16.53$G & $1.8\%$ & $2.31e^{13}$ & $13.36$G & $2.1\%$ & $1.51e^{14}$ & $14.31$G \\
\midrule
\textbf{Algorithm \ref{alg_lcc}} & $\mathbf{16}$ & $\mathbf{0.2\%}$ & $\mathbf{3.49e^{13}}$ & $\mathbf{16.76}$G & $\mathbf{0.2\%}$ & $\mathbf{2.12e^{13}}$ & $\mathbf{13.27}$G & $\mathbf{0.1\%}$ & $\mathbf{8.67e^{13}}$ & $\mathbf{14.44}$G \\
\bottomrule
\end{tabular}}
\caption{We report the relative error, computational cost (FLOPs), and memory usage on MMLU, Addition datasets, and GSM8K. The relative error is computed against actual inference results.}\label{tab_error_baseline_main}
\end{table*}

\begin{table*}[t]
\centering
\resizebox{0.995\textwidth}{!}
{
\begin{tabular}{@{}l|cccc|cccc|cc@{}}
\toprule
Category & \multicolumn{4}{c|}{Sentiment Analysis} & \multicolumn{4}{c|}{Math Reasoning} & \multicolumn{2}{c}{Graph Reasoning} \\
Dataset & \multicolumn{2}{c}{SST-2} & \multicolumn{2}{c|}{Poem Sentiment} & \multicolumn{2}{c}{Modular Addition} & \multicolumn{2}{c|}{Coin Flip} & \multicolumn{2}{c}{Edge Existence} \\ 
Metric & Acc. & FLOPs & Acc. & FLOPs & Acc. & FLOPs & Acc. & FLOPs & Acc. & FLOPs \\
\midrule
Random-$k$      & $74.7_{\pm4.5}$ & $1.92e^{14}$ & $57.3_{\pm4.5}$ & $4.17e^{14}$ & $53.2_{\pm1.7}$ & $1.68e^{14}$ & $55.3_{\pm7.5}$ & $2.34e^{14}$ & $48.9_{\pm3.6}$ & $6.55e^{14}$\\
BM25            & $77.8_{\pm1.9}$ & $1.92e^{14}$ & $55.7_{\pm2.8}$ & $4.17e^{14}$ & $58.3_{\pm2.5}$ & $1.68e^{14}$ & $42.3_{\pm1.1}$ & $2.34e^{14}$ & $51.2_{\pm0.6}$ & $6.55e^{14}$\\
Top-$k$         & $88.9_{\pm1.2}$ & $1.92e^{14}$ & $63.8_{\pm1.5}$ & $4.17e^{14}$ & $60.3_{\pm1.1}$ & $1.68e^{14}$ & $45.5_{\pm1.4}$ & $2.34e^{14}$ & $58.3_{\pm1.1}$ & $6.55e^{14}$\\
Top-$k$ + Alg.~\ref{alg_lcc}   & $88.7_{\pm0.8}$ & $9.14e^{12}$ & $63.4_{\pm1.0}$ & $2.17e^{13}$ & $60.3_{\pm1.1}$ & $7.51e^{12}$ & $45.4_{\pm0.8}$ & $9.56e^{12}$ & $58.0_{\pm1.3}$ & $2.72e^{13}$\\
\textsc{GradCE} & $83.5_{\pm0.7}$ & $6.71e^{13}$ & $70.3_{\pm0.9}$ & $1.47e^{14}$ & $64.5_{\pm2.1}$ & $5.80e^{13}$ & $68.2_{\pm0.8}$ & $7.80e^{13}$ & $65.4_{\pm1.3}$ & $2.61e^{14}$\\
BRIDGE          & $89.9_{\pm0.3}$ & $5.73e^{15}$ & $73.3_{\pm0.8}$ & $7.12e^{15}$ & $70.8_{\pm1.3}$ & $7.01e^{15}$ & $82.3_{\pm0.6}$ & $1.00e^{16}$ & $76.3_{\pm0.4}$ & $2.84e^{16}$\\
BRIDGE + Alg.~\ref{alg_lcc}    & $89.9_{\pm0.5}$ & $4.93e^{14}$ & $73.0_{\pm0.9}$ & $5.70e^{14}$ & $70.1_{\pm0.5}$ & $5.49e^{14}$ & $82.2_{\pm0.8}$ & $7.22e^{14}$ & $75.8_{\pm0.8}$ & $2.08e^{15}$\\\midrule
\textbf{Algorithm \ref{alg_aslcc}}  & $\mathbf{95.9_{\pm 0.7}}$ & $\mathbf{4.52e^{14}}$ & $\mathbf{78.6_{\pm 1.1}}$ & $\mathbf{5.23e^{14}}$ & $\mathbf{81.8_{\pm 1.2}}$ & $\mathbf{5.08e^{14}}$ & $\mathbf{85.3_{\pm1.2}}$ & $\mathbf{6.61e^{14}}$ & $\mathbf{83.6_{\pm0.6}}$ & $\mathbf{1.81e^{15}}$\\
\bottomrule
\end{tabular}}
\caption{We report the test accuracy (\%) and the computational cost (FLOPs) across five different kinds of datasets. We compare our approach with several existing demonstration selection methods, using $50$ demonstrations. We run each experiment with three random seeds to report the standard deviations.}\label{tab_demonstration_baseline}
\end{table*}

\section{Experiments}\label{sec_main_experiment}

We evaluate our method by addressing the following three questions. First, how accurately and efficiently does Algorithm~\ref{alg_lcc} approximate full-context inference compared with existing methods? Second, how effective is Algorithm~\ref{alg_aslcc} in terms of the accuracy and computational cost of downstream demonstration selection? Third, how well does the learned SSM scale to much longer contexts and transfer across different domains?

Through extensive experiments in both text and reasoning datasets, we provide positive answers to all three questions above.
Finally, we conduct detailed ablation studies to validate the design of our approach.

\subsection{Experiment Setup}

\textbf{Datasets.} We evaluate our approach across a diverse set of downstream tasks with different output formats. SST-2 and Poem Sentiment require categorical sentiment labels. Coin Flip and Edge Existence require binary decisions. Modular Addition requires numerical answers.

We further evaluate the accuracy of Algorithm~\ref{alg_lcc} on MMLU and GSM8K. MMLU covers knowledge-intensive multiple-choice questions across a broad range of academic and professional subjects, and GSM8K requires the model to generate multi-step reasoning processes and final numerical answers. Thus, our evaluation covers both short-output classification tasks and longer-form chain-of-thought reasoning tasks.

\paragraph{Baselines.} Regarding efficient inference methods, we consider full inference (dense), streaming LLM~\citep{xiao2024efficient}, LM-infinite~\citep{han2024lm}, Gist tokens~\citep{mu2023learning}, in-context autoencoder (ICAE)~\citep{ge2024context}, block-sparse attention (BSA)~\citep{xiao2025efficient}, and DuoAttention~\citep{xiao2025duoattention} as baselines.

Then, we consider baselines that use different measures to rank the demonstrations. These include random selection (Random-$k$), selection based on probabilistic relevance rankings (BM25), embedding similarities (top-$k$), Bayesian refinement, and iterative demonstration generation, for example (BRIDGE)~\citep{wan2025few}, cross-entropy-based selection with gradient estimation (\textsc{GradCE})~\citep{zhang2025linear}. We defer a detailed description of datasets, models, and baselines to Appendix~\ref{app_detailed_setup}.
All other implementation details are deferred to Appendix \ref{app_exp}.

\subsection{Experiment Results}

\paragraph{Accuracy and efficiency of Algorithm~\ref{alg_lcc}.} We evaluate our approach using the Qwen-3B model with $50$ demonstrations for MMLU and Modular Addition and $20$ for GSM8K. We compute the relative error of the output logits, the computational cost (measured in the number of FLOPs), and the peak memory usage for each baseline.

As shown in Table~\ref{tab_error_baseline_main}, Algorithm \ref{alg_lcc} reduces the inference cost by approximately two orders of magnitude compared to Dense inference. By compressing the global prefix into $16$ tokens, our method maintains a relative output error below $0.2\%$ across all tasks. Compared to the baseline BSA, our method yields an average error reduction of $2.17\%$ and a $35.5\%$ decrease in FLOPs, with minimal memory overhead. We defer the results on the remaining datasets and those using Qwen-1.5B to Table~\ref{tab_omitted_error_1.5B} and Table~\ref{tab_omitted_baseline_error_1.5B_3B}.

\paragraph{Accuracy and efficiency of Algorithm~\ref{alg_aslcc}.} We integrate our approach into existing selection frameworks to assess its compatibility and evaluate it across five datasets. We evaluate the accuracy and FLOPs of the whole pipeline in Algorithm~\ref{alg_aslcc}.

Using the Qwen-3B model with $50$ demonstrations, Table~\ref{tab_demonstration_baseline} shows that our approach accelerates existing pipelines, with an average accuracy drop of $0.22\%$ across other baselines. Compared to the baseline, BRIDGE, our approach yields an average accuracy improvement of $6.48\%$ while requiring much fewer FLOPs. We analyze the efficiency of Algorithm~\ref{alg_aslcc} in Appendix~\ref{app_omitted_results}.

\paragraph{Length and task generalization.}
To evaluate length generalization, we test the module, trained solely on $50$ demonstrations, on extended contexts containing up to $1000$ demonstrations. As shown in Figure~\ref{fig_k_generation}, our method maintains stable performance without additional fine-tuning. When increasing the number of demonstrations to $1000$, the relative output error remains below $1.8\%$ across all datasets.

\begin{figure}[t]
\centering
\begin{subfigure}{0.24\textwidth}
\centering
\includegraphics[width=0.99\textwidth]{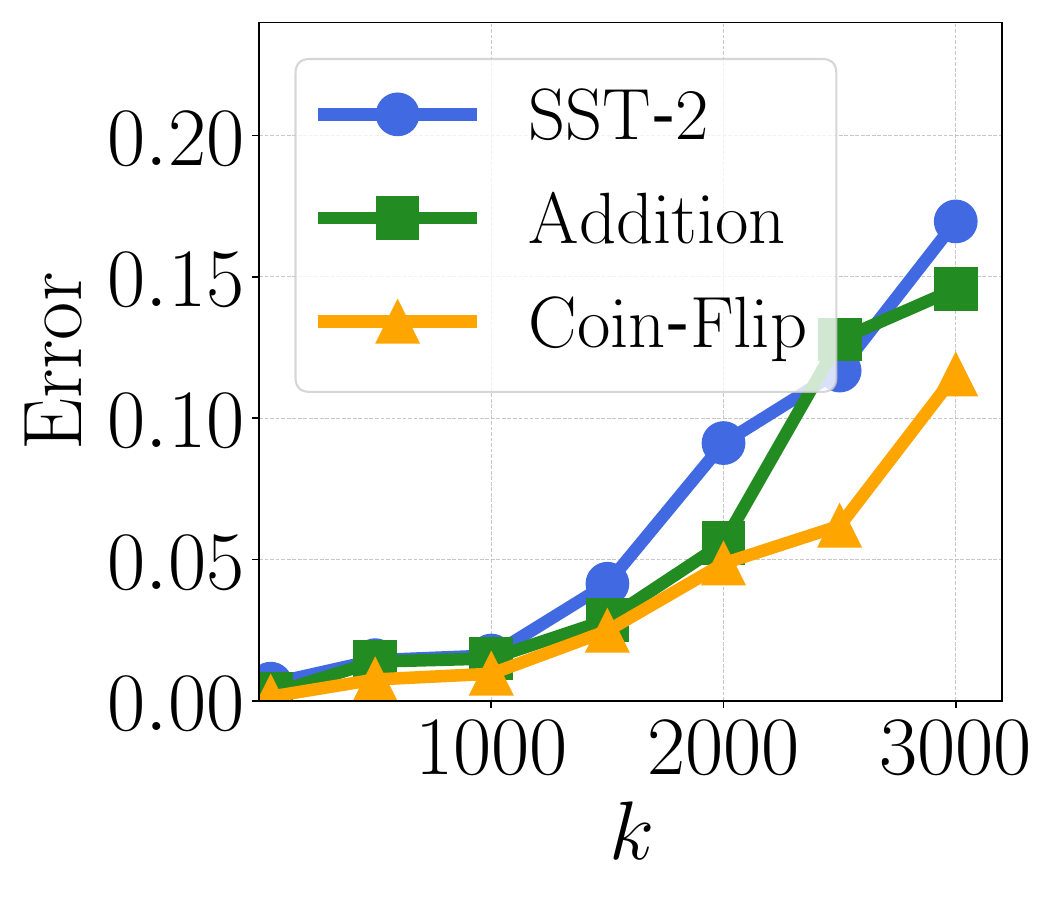}
\caption{Length error}
\label{fig_k_generation}
\end{subfigure}\hfill
\begin{subfigure}{0.24\textwidth}
\centering
\includegraphics[width=0.99\textwidth]{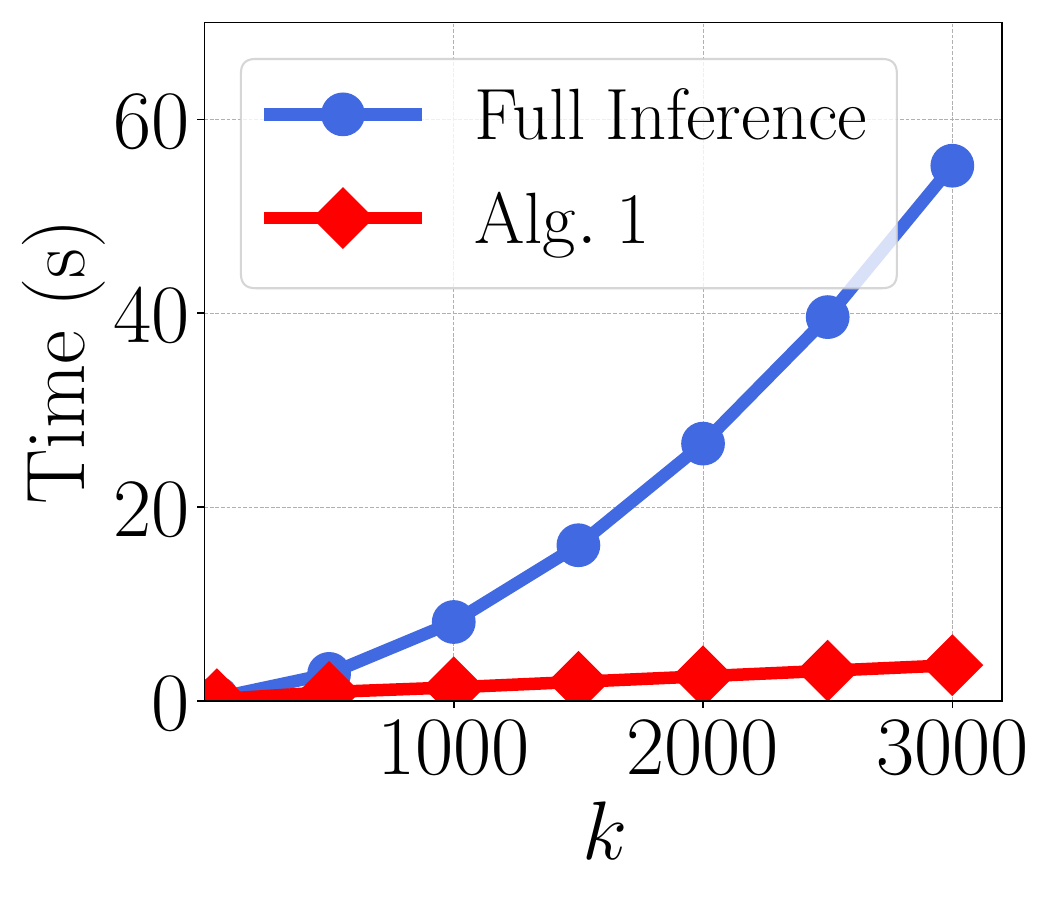}
\caption{Runtime scaling}
\label{fig_k_clock_time}
\end{subfigure}
\caption{Figure~\ref{fig_k_generation} reports the relative output error as the number of demonstrations $k$ increases. Figure~\ref{fig_k_clock_time} reports the wall-clock inference time as $k$ increases. 
}
\end{figure}

We further compare the wall-clock time for processing a single query with full inference and our approach. As shown in Table~\ref{tab_ablation} on the left, full inference retains the quadratic growth as the context becomes longer, while our approach scales approximately linearly.

Finally, we evaluate cross-task transfer across SST-2, MMLU, Modular Addition, and GSM8K. When adapting to a new task, we initialize it using the SSM trained on the source task and apply only the second-stage fine-tuning objective to the target task.

Table~\ref{tab_ablation} on the right reports the relative error for each source-target pair, where rows indicate the source training tasks and columns indicate the target evaluation tasks. Adaptation takes $65\%$ fewer GPU hours than full training, and the relative error is lower than $1.0\%$ for all task pairs.

\subsection{Ablation Studies}

We conduct ablation studies to examine the main components of our framework.

\paragraph{Number of groups.} First, we study the effects of the HiPPO state dimension and the number of SSM layer groups on SST-2. As shown in Table~\ref{tab_ablation}, increasing either value reduces the test loss, but the improvement gradually becomes smaller. In particular, the effect of increasing the HiPPO state dimension largely saturates at $512$.

\paragraph{Number of tokens.} Then, we study the effect of the number of tokens, $n_v$, using Qwen-3B with $50$ demonstrations. We vary $n_v$ from $2$ to $32$ and report the output mean squared error (MSE) and the accuracy drop relative to full-context inference. As shown in Table~\ref{tab_virtual_tokens}, increasing $n_v$ generally reduces both metrics. The results become stable at $n_v=16$. Thus, we use $n_v=16$ in the main experiments.

\begin{table}[t]
\begin{subfigure}{0.25\textwidth}
\centering
\includegraphics[width=0.98\textwidth]{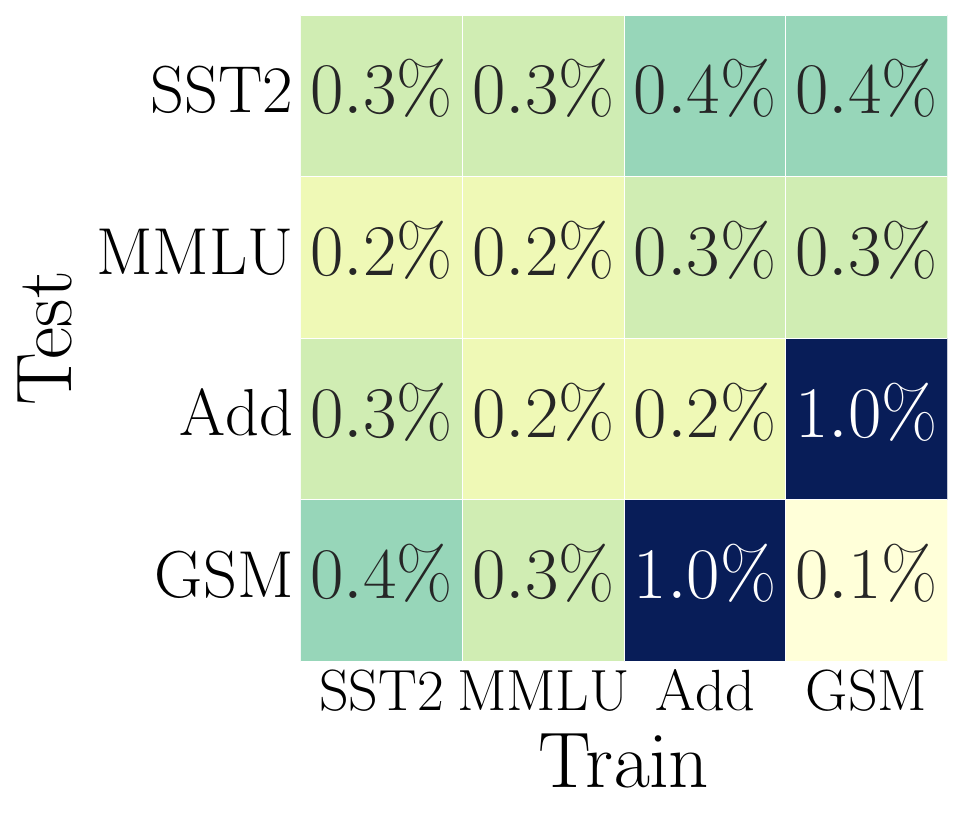}
\end{subfigure}\hfill
\begin{subfigure}{0.23\textwidth}
\centering
\includegraphics[width=0.95\textwidth]{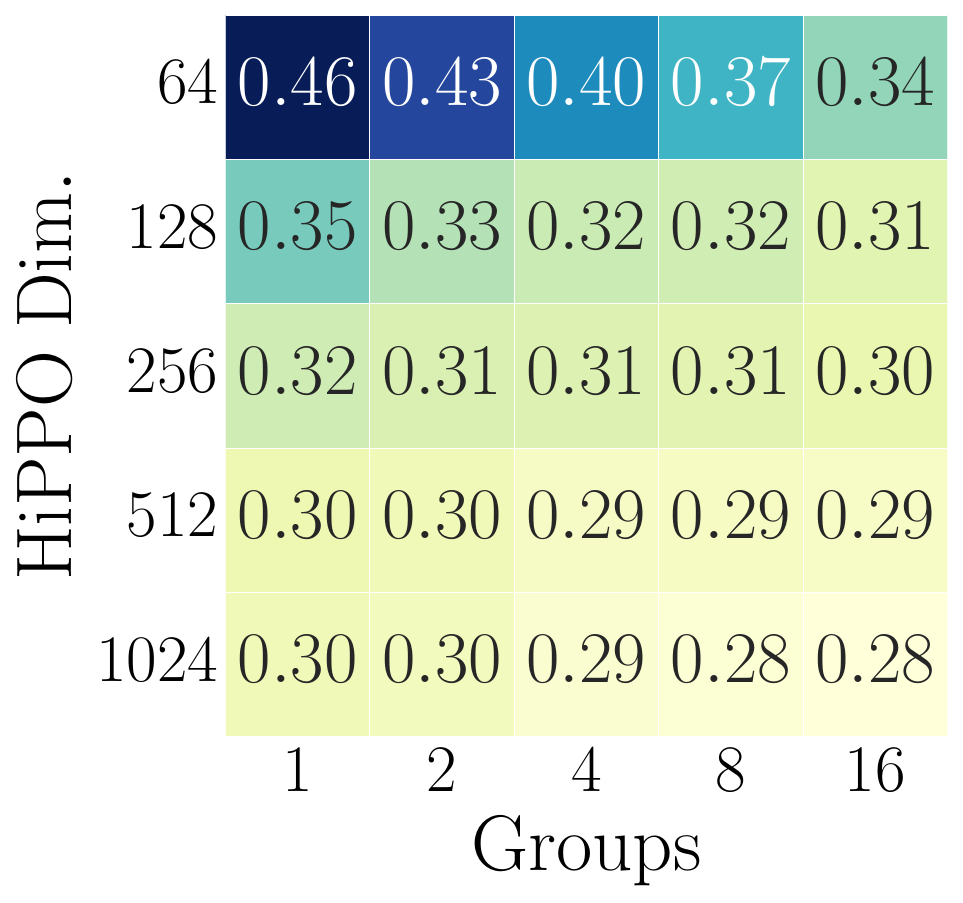}
\end{subfigure}
\caption{Left: cross-domain transfer error between training and test tasks. Right: test loss under different HiPPO state dimensions and numbers of layer groups.}\label{tab_ablation}
\end{table}

\paragraph{Empirical sample size scaling.} Finally, we study the empirical sample complexity of Algorithm~\ref{alg_aslcc}. First, we study the empirical sample complexity of the randomized ensemble method using linear functions~\citep{garg2022can}. Following the setting of~\citet{garg2022can}, we construct a mixed candidate pool containing demonstrations from two distributions, $\mathcal{D}_1$ and $\mathcal{D}_2$, while sampling all validation queries from $\mathcal{D}_1$. We vary both the number of demonstrations from $\mathcal{D}_1$ in each subset and the number of subsets sampled by the ensemble. The results show that each subset requires more than $20$ demonstrations from $\mathcal{D}_1$ to achieve near-zero validation loss. Moreover, sampling more than $50$ subsets produces a clear separation between the average affinity scores assigned to demonstrations from $\mathcal{D}_1$ and $\mathcal{D}_2$. Table~\ref{tab_num_subsets} in Appendix \ref{app_omitted_results} reports the full results.

\begin{table*}[t]
    \centering
    {\small
    \begin{tabular}{llccccc}
    \toprule
    Dataset & Metric & $n_v=2$ & $n_v=4$ & $n_v=16$ & $n_v=32$ \\
    \midrule
    \multirow{2}{*}{SST-2} & MSE & $0.027$ & $0.015$ & $0.006$ & $0.003$ \\ & Accuracy drop & $0.07$ & $0.05$ & $0.01$ & $0.01$ \\
    \midrule
    \multirow{2}{*}{Poem Sentiment} & MSE & $0.010$ & $0.006$ & $0.002$ & $0.002$ \\ & Accuracy drop & $0.06$ & $0.05$ & $0.04$ & $0.04$ \\
    \midrule
    \multirow{2}{*}{GSM8K} & MSE & $0.002$ & $0.001$ & $0.002$ & $0.001$ \\ & Accuracy drop & $0.02$ & $0.02$ & $0.01$ & $0.01$ \\
    \bottomrule
    \end{tabular}}
    \caption{Varying number of tokens using Qwen-3B with $50$ demonstrations. We report the output MSE and the accuracy drop relative to full-context inference on SST-2, Poem Sentiment, and GSM8K.}\label{tab_virtual_tokens}
\end{table*}

\begin{table*}[t!]
\centering
{\small
\begin{tabular}{@{}lcccc@{}}
\toprule
\# Subsets & $10$ & $50$ & $100$ & $200$ \\ 
\midrule
Samples from $\mathcal{D}_1$ & $0.22_{\pm0.04}$ & $0.55_{\pm0.03}$ & $0.57_{\pm0.04}$ & $0.56_{\pm0.03}$  \\
Samples from $\mathcal{D}_2$ & $0.21_{\pm0.01}$ & $0.52_{\pm0.04}$ & $0.51_{\pm0.06}$ & $0.48_{\pm0.04}$  \\
\bottomrule
\end{tabular}}
\caption{To illustrate the scoring mechanism in our random ensemble method, we report the average score of demonstration samples from $\mathcal{D}_1$ and $\mathcal{D}_2$, given the query set generated from $\mathcal{D}_1$. We vary the number of subsets and find that the in-distribution samples always achieve a higher score than out-of-distribution samples.}\label{tab_num_subsets}
\end{table*}

\section{Related Work}\label{sec_related}

\textbf{Long-context model inference.} Work on reducing transformer cost can be grouped into four lines. First, architectural methods use SSMs~\citep{gu2020hippo,gu2022efficiently} and hybrid mamba-transformer designs~\citep{gu2024mamba,ren2025samba,oncescu2025flash} to achieve linear-time long-sequence modeling. 
For long-convolution sequence models, the fast Fourier transform~\citep{oncescu2025flash} can be utilized to perform exact autoregressive inference efficiently.
Note that these approaches target alternative architectures and cannot be directly applied to transformer-based generative models.
The second line of work~\citep{xiao2024efficient,han2024lm} limits each query to a sliding window of recent tokens and a few fixed initial tokens that act as attention sinks at inference time. Third, memory-token methods map long contexts into learned slots~\citep{mu2023learning,chevalier2023adapting,ge2024context}, but they still need to run a transformer over the original input.
Finally, previous works~\cite{yang2024parallelizing,yang2026path,dao2022flashattention} use blockwise computation to improve hardware efficiency for long sequences. Given their insights, we apply the same general strategy to the fixed HiPPO recurrence by combining the state updates within each chunk into a single batched operation.

\paragraph{Demonstration selection and prompt optimization.} In-context learning relies heavily on the quality and structure of the provided demonstrations~\citep{min2022rethinking}. Traditional selection relies on static embedding retrieval, which ignores the model's actual output. To capture these combinatorial effects, recent methods incorporate model inference feedback. For instance, \citet{wan2025few} alternate between selecting influential demonstrations and using them to regenerate an improved many-shot prompt. Since running the full model on every candidate subset is expensive, \citet{zhang2025linear} estimate subset-level outputs from input gradients and reduce the cost of evaluating many subsets.

Another related line of work involves surrogate modeling \cite{li2023identification}, task attribution \cite{zhang2025scalable,zhang2026efficient}, and influence estimation. 
A theoretical analysis~\citep{yang2025precise} shows that the benefit of combining datasets depends jointly on task mismatch and source sample size. Adding more source data can even turn positive transfer into negative transfer.
\citet{li2023identification,li2024scalable} use surrogate modeling to efficiently identify negative transfer among data, and \citet{zhang2026efficient} use kernel surrogate modeling to estimate affinity. The same goal of estimating interactions among datasets has also been studied for large language models. Gradient-based methods predict the fine-tuning performance of different dataset combinations without training a separate model for each combination~\citep{li2024gradex}, and these estimates can then guide dataset grouping~\citep{li2025efficient}.
Broadly, estimated task affinity can guide not only dataset grouping but also how model components are shared across tasks~\citep{li2026efficiently}.

Beyond gradient-based estimation and heuristic scoring, selecting efficient candidate data points connects to the broader literature on experimental design and automated trajectory optimization from historical data \citep{xiong2024optimal, xiong2024data, xiong2025automated, zhang2025scalable}. It would be interesting to connect ideas from this paper to automated trajectory optimization from historical data.

\section{Conclusion}\label{sec_conclude}
We address the computational bottleneck of demonstration selection with SSMs, where many candidate subsets require repeated transformer inference. We train the SSMs to estimate the transformer output for different demonstration subsets. It reduce the inference complexity to linear. Then, we introduce a random selection method that compares many subsets and identifies useful demonstrations based on the SSMs inference output.

\section*{Acknowledgment}

Thanks to the anonymous reviewers and the area chair for their constructive feedback.
The work of Ziniu Zhang and Zhenshuo Zhang is partially funded by NSF award IIS-2412008, a startup fund from Northeastern University, and Khoury PhD fellowships.

%% file: appendix.tex
\section{Proof of Proposition \ref{prop_compression}}\label{app_group_error}

\begin{proof}
For completeness, the HiPPO-LegS transition matrix used by each layer group is defined as
\begin{equation}
\bar{A}_{p,q}=
\begin{cases}
-\sqrt{(2p+1)(2q+1)}, & \text{if } p>q,\\
-(p+1), & \text{if } p=q,\\
0, & \text{if } p<q,
\end{cases}
\end{equation}

Let $\{\phi_r\}_{r=1}^{T}\subset\mathbb{R}^{T}$ denote an orthonormal discrete Legendre basis for the token dimension, ordered by increasing polynomial degree, so
\[
\phi_r^\top\phi_{r'}=\mathbbm{1}\{r=r'\}.
\]
For each layer group $G_i$, write $C^{(i)}\in\mathbb{R}^{T\times D_i}$, where the rows correspond to token positions and the columns contain the key and value states from all layers and key-value heads in the group. The vector of the $r$-th HiPPO-LegS coefficients is
\[
c_r^{(i)}
:=
\bigl(C^{(i)}\bigr)^\top\phi_r
\in\mathbb{R}^{D_i}.
\]
Therefore,
\[
C^{(i)}
=
\sum_{r=1}^{T}
\phi_r\bigl(c_r^{(i)}\bigr)^\top.
\]
Since $\Pi_N$ projects along the token dimension onto the first $N$ basis functions, $\Pi_N C^{(i)}=\sum_{r=1}^{N}\phi_r\bigl(c_r^{(i)}\bigr)^\top$.
The omitted part of the cache is thus
\[
C^{(i)}-\Pi_N C^{(i)}
=
\sum_{r=N+1}^{T}
\phi_r\bigl(c_r^{(i)}\bigr)^\top.
\]
By orthonormality of the basis,
\begin{equation}
\left\|
C^{(i)}-\Pi_N C^{(i)}
\right\|_F^2
=
\sum_{r=N+1}^{T}
\left\|c_r^{(i)}\right\|_2^2.
\label{eq:hippo_tail_energy}
\end{equation}
Applying this inequality to each term in
\eqref{eq:hippo_tail_energy} gives
\[
\begin{aligned}
&\left\|
C^{(i)}-\Pi_N C^{(i)}
\right\|_F^2 \\
\le&
\frac{1}{(N+1)^{2s}}
\sum_{r=N+1}^{T}
r^{2s}
\left\|c_r^{(i)}\right\|_2^2,
\end{aligned}
\]
since $\frac{r^{2s}}{(N+1)^{2s}} \ge 1$ for $r\ge N+1$ and $s>0$.
\end{proof}

\section{Omitted Experiments}\label{app_exp}

\subsection{Implementation Setup}
\label{app_detailed_setup}
\textit{Datasets.} We evaluate our method on datasets covering several task types, including language understanding, algorithmic reasoning, arithmetic reasoning, and graph reasoning. The number of demonstrations within each task is $1500$.

The \href{https://huggingface.co/datasets/nyu-mll/glue/viewer/sst2}{SST-2} dataset is a binary sentiment classification benchmark from \href{https://huggingface.co/datasets/nyu-mll/glue}{GLUE}, consisting of movie reviews labeled as positive or negative.

The \href{https://huggingface.co/datasets/google-research-datasets/poem_sentiment}{Poem Sentiment} dataset contains sentiment labels for lines of poetry, with three classes: positive, neutral, and negative. This dataset is used to evaluate sentiment understanding in literary text.

Modular Addition is an arithmetic reasoning task in which the model is given two integers $a,b \in \{0,\ldots,P-1\}$ for a prime $P$ and is asked to predict their sum modulo $P$, that is, to predict $c$ such that $a+b \equiv c \pmod{P}$.

The \href{https://huggingface.co/datasets/skrishna/coin_flip}{Coin Flip} dataset is a reasoning task in which the model is given a natural language description of a sequence of fair coin flips and must predict the final outcome, either heads or tails.

The Edge Existence dataset from \href{https://github.com/google-research/google-research/tree/master/graphqa}{GraphQA} is a graph reasoning task in which the model is given an undirected graph and must determine whether a specified edge is present in the graph.

The \href{https://huggingface.co/datasets/cais/mmlu}{MMLU} dataset is a multitask benchmark covering a wide range of subjects, including humanities, social sciences, and STEM. It evaluates a model’s ability to perform multiple-choice question answering across diverse domains. We use the college-level tasks in the evaluation.

The \href{https://huggingface.co/datasets/openai/gsm8k}{GSM8K} dataset is a grade-school math reasoning benchmark consisting of natural language word problems. Each example requires multi-step arithmetic reasoning, and the model is asked to produce the final numerical answer. We use this dataset to evaluate whether Algorithm~\ref{alg_lcc} preserves the reasoning information needed for long-context mathematical inference.

\textit{Models.} Our method is a model-agnostic method. In our experiments, we evaluate it on a diverse set of open-source language models that vary in scale and architecture, including \href{https://huggingface.co/Qwen/Qwen2.5-1.5B-Instruct}{Qwen2.5-1.5B-Instruct}, \href{https://huggingface.co/Qwen/Qwen2.5-3B-Instruct}{Qwen2.5-3B-Instruct}, \href{https://huggingface.co/Qwen/Qwen2.5-7B-Instruct}{Qwen2.5-7B-Instruct}, and \href{https://huggingface.co/meta-llama/Llama-3.1-8B-Instruct}{Llama-3.1-8B-Instruct}. This setup allows us to assess the effectiveness and scalability of our method across models with different capacities, training paradigms, and usage settings.

\textit{Training budget.} We train our method on a single NVIDIA RTX A6000 GPU. We set the number of virtual tokens to $16$, the SSM state dimension to $512$, and the number of layer groups to $4$. Training for $1$ epoch in phase one and $2$ epochs in phase two requires $1.15$ GPU hours.

\textit{Baselines.} We compare Algorithm~\ref{alg_lcc} against full inference, StreamingLLM, LM-Infinite, Block Sparse Attention baselines, and DuoAttention. Then, we compare Algorithm~\ref{alg_aslcc} against baselines including BM25, top-$k$, BRIDGE, \textsc{GradCE}.

Full Attention serves as our baseline. Every token attends to all preceding tokens, retaining the complete Key-Value cache across all layers and heads. While preserving maximum model capability, its linear memory scaling and quadratic pre-filling cost become prohibitively expensive for large demonstration pools.

StreamingLLM \citep{xiao2024efficient} exploits the phenomenon of attention sinks, where initial tokens receive disproportionately high attention scores regardless of semantic relevance. It maintains a fixed-size KV cache comprising only these initial sink tokens and a sliding window of recent tokens, evicting all intermediate context. This achieves constant memory usage but inherently degrades performance on tasks requiring fact retrieval from the middle of the input.

LM-Infinite \citep{han2024lm} addresses length generalization beyond the model's pre-training window. To prevent attention degradation from out-of-distribution position indices, it discards tokens outside a prescribed local window and applies a $\Lambda$-shaped position remapping. This aligns the remaining tokens' indices with the training distribution, enabling infinite context processing without fine-tuning. However, similar to StreamingLLM, it suffers from information loss by permanently evicting middle-context tokens.

DuoAttention \citep{xiao2025duoattention} categorizes attention heads into retrieval heads (attending to arbitrary positions) and streaming heads (attending to sinks and recent tokens). During a lightweight identification phase, a trainable gate $\alpha_{i,j}$ is assigned to each KV head and optimized via synthetic passkey-retrieval data. This minimizes output deviation from Full Attention under an $\ell_1$ sparsity penalty. At deployment, binarized gates selectively apply the StreamingLLM eviction policy exclusively to streaming heads, while retrieval heads retain the complete KV cache. This selectively reduces memory and latency while preserving long-range retrieval capabilities.

ICAE~\citep{ge2024context} compresses the input context into a small number of learned memory tokens through an autoencoding objective. It first encodes the original context into compact latent representations, and then conditions the language model on these memory tokens to recover the information needed for downstream generation. 

Implicit In-context Learning (I2CL)~\citep{li2025implicit} moves demonstrations from the token space into the model's activation space. It extracts a vector from each demonstration, combines these vectors in a way that is independent of their order, and injects the resulting context vector into the model's residual streams during inference. I2CL provides few-shot performance with inference cost close to that of zero-shot prediction. However, constructing its context vector still requires Transformer forward passes over the demonstrations. 

The top-$k$ method selects the $k$ most similar candidates based on feature similarity, given a test input. We compute the cosine similarity in the top-$k$ range between the last-layer hidden representations of queries and demonstration examples. The BM25 method utilizes a term frequency-based ranking function to retrieve the top-$k$ candidates whose input texts are most relevant to the query, with a focus on lexical overlap rather than embedding-level similarity.
\textsc{GradCE}~\citep{zhang2025linear} first filters candidates by embedding similarity, then selects a random anchor and performs a single forward and backward pass on all validation samples to precompute the loss and gradient with respect to the input embedding at that anchor. For each remaining candidate, it computes the embedding difference from the anchor and takes the inner product with the precomputed gradient to approximate the candidate's validation loss via a first-order Taylor expansion, selecting the $k$ candidates with the lowest estimated loss. BRIDGE~\citep{wan2025few} iteratively alternates between using Bayesian optimization to select a small high-performing subset of ICL demonstrations from the candidate pool ("optimize") and using that subset as seed examples to re-generate reasoning paths on the training set back to the many-shot regime ("generate"), repeating this process over multiple rounds to progressively improve demonstration quality.

\begin{table*}[t]
\centering
{\small
\begin{tabular}{@{}llcccc@{}}
\toprule
Model & \# Demonstrations & SST-2 & CR & Poem Sentiment & Modular Addition \\ \midrule
\multirow{1}{*}{Qwen-1.5B} & $k=25$ & $2.3_{\pm0.2}\times10^{-3}$ & $1.9_{\pm0.3}\times10^{-3}$ & $8.9_{\pm2.5}\times10^{-4}$ & $1.1_{\pm0.1}\times10^{-3}$ \\\midrule
\multirow{1}{*}{Qwen-3B} & $k=25$ & $5.3_{\pm0.2}\times10^{-3}$ & $3.2_{\pm0.3}\times10^{-3}$ & $2.0_{\pm2.5}\times10^{-3}$ & $3.5_{\pm0.1}\times10^{-3}$ \\
\bottomrule
\end{tabular}}
\caption{Relative error of the predictive output between the full demonstration prefix $p$ and the compressed virtual tokens $\hat p$. Evaluations are reported across multiple downstream tasks for varying demonstration counts. Lower values indicate higher fidelity to the original uncompressed model output. We run three times to compute the mean and standard deviations.}\label{tab_error_k=25}
\end{table*}

\subsection{Efficient Implementation}

Previous work~\cite{xiao2024efficient} shows that standard LLM inherently relies on initial tokens as ``attention sinks'' to stabilize the generation process. Thus, we explicitly retain a small number of exact initial KV states as sink tokens from the raw prefix. These states are concatenated with our generated virtual tokens to form the compressed context. 

The standard SSM recurrence processes tokens sequentially:
\begin{equation}
    s_t = A s_{t-1} + u_t,  u_t = B x_t, s_0 = 0,
\end{equation}
where $x_t \in \mathbb{R}^H$ is the input embedding at step $t$, $B \in \mathbb{R}^{D \times H}$ is a learned input projection, $A \in \mathbb{R}^{D \times D}$ is the frozen discretized HiPPO matrix, and $s_t \in \mathbb{R}^D$ is the hidden state. A naive implementation requires $T$ sequential kernel launches, which may under utilize the GPU.

We observe that unrolling the recurrence over $C$ consecutive steps yields a closed-form expression. Starting from an incoming state $s_{\mathrm{in}}$, the state after processing inputs $u_0, u_1, \ldots, u_{C-1}$ is:
\begin{equation}\label{eq_chunk-closed-form}
    s_{\mathrm{out}} = A^{C}  s_{\mathrm{in}} + \sum_{t=0}^{C-1} A^{C-1-t} u_t.
\end{equation}
The summation can be rewritten as a single batched contraction. Define the weight matrices $W_t = A^{C-1-t}$ for $t = 0, \ldots, C{-}1$, so that $W_0 = A^{C-1}$ (oldest input, most decay) and $W_{C-1} = I$ (newest input, no decay). Stacking these into a tensor $\mathbf{W} \in \mathbb{R}^{C \times D \times D}$ and the chunk inputs into $\mathbf{U} \in \mathbb{R}^{C \times D}$, the entire summation reduces to:
\begin{equation}
    \sum_{t=0}^{C-1} W_t u_t = \sum_{t=0}^{C-1} \mathbf{W}[t] \mathbf{U}[t],
\end{equation}
which is evaluated as a single \texttt{einsum} operation on the GPU.

We partition the full prefix of length $T$ into $\lceil T/C \rceil$ contiguous chunks. Processing proceeds as follows: for each chunk $k = 0, 1, \ldots$, we first compute the powers $\{A^{C-1}, A^{C-2}, \ldots, I\}$ iteratively using $C$ matrix multiplications of size $D \times D$. Then, we evaluate the within-chunk contribution via the batched contraction. Finally, we advance the inter-chunk state as $s_{\mathrm{in}}^{(k+1)} = A^C \, s_{\mathrm{in}}^{(k)} + \text{contrib}^{(k)}$. Since $A$ is frozen, the powers are computed on the fly and discarded after each chunk to avoid storing a $C \times D \times D$ buffer.

\begin{table*}[ht!]
    \centering
    {\small\begin{tabular}{lccccc}
    \toprule
    Approach & SST-2 & Poem Sentiment & Addition & Coin Flip & Edge Existence \\
    \midrule
    Algorithm~\ref{alg_aslcc} & $95.9_{\pm 0.7}$ & $78.6_{\pm 1.1}$ & $81.8_{\pm 1.2}$ & $85.3_{\pm 1.2}$ & $83.6_{\pm 0.6}$ \\
    Algorithm~\ref{alg_aslcc} without context compression & $95.0_{\pm 0.5}$ & $79.7_{\pm 0.8}$ & $82.8_{\pm 1.0}$ & $85.3_{\pm 1.2}$ & $83.7_{\pm 1.3}$ \\
    \bottomrule
    \end{tabular}}
    \caption{Test accuracy (\%) of Algorithm~\ref{alg_aslcc} with and without context compression. Each result reports the mean and standard deviation over three random seeds.}\label{tab_selection_wo_compression}
\end{table*}

\begin{table*}[t]
\centering
\resizebox{1.0\textwidth}{!}
{\small
\begin{tabular}{@{}l|c|ccc|ccc|ccc@{}}
\toprule
\multirow{2}{*}{Approach} & Window & \multicolumn{3}{c|}{MMLU} & \multicolumn{3}{c|}{Modular Addition} & \multicolumn{3}{c}{GSM8K} \\ 
& Length & Error & FLOPs & Memory & Error & FLOPs & Memory & Error & FLOPs & Memory \\\midrule
\multirow{2}{*}{StreamingLLM}   
& $32$ & $2.5\%$ & $7.72e^{13}$ & $15.88$G  & $10.7\%$ & $1.90e^{13}$ & $13.95$G & $6.4\%$ & $1.25e^{14}$ & $14.05$G \\
                                & $128$ & $0.7\%$ & $7.75e^{13}$ & $15.95$G & $8.7\%$ & $1.91e^{13}$ & $13.96$G  & $5.1\%$ & $1.26e^{14}$ & $14.07$G \\\midrule
\multirow{2}{*}{LM-Infinite}     & $32$ & $11.6\%$ & $7.73e^{13}$ & $15.90$G  & $14.3\%$ & $1.90e^{13}$ & $12.96$G & $5.7\%$ & $1.26e^{14}$ & $14.06$G \\
                                & $128$ & $5.1\%$ & $7.75e^{13}$ & $15.98$G  & $8.7\%$ & $1.91e^{13}$ & $12.99$G & $4.9\%$ & $1.27e^{14}$ & $14.07$G \\
\bottomrule
\end{tabular}}
\caption{We report the results using more window lengths on MMLU, Addition datasets, and GSM8K with $k=50$. The relative error is computed against actual inference results.}\label{tab_error_baseline}
\end{table*}

This reduces the number of sequential state updates from $T$ to $\lceil T/C \rceil$, while each chunk-level operation is a single, well-parallelised GPU kernel. The underlying SSM dynamics are unchanged---the final state $s_T$ is mathematically identical to that of the token-by-token scan.

During training, we observe that directly optimizing the distillation objectives (logits and hidden states) without prior initialization degrades performance. The root cause is a bootstrap problem: randomly initialized virtual KV tokens produce degenerate attention patterns in the frozen Transformer, which in turn yield near-uniform output distributions. The resulting logit-level and hidden-state gradients carry little useful signal for the auxiliary, preventing meaningful learning. The first-stage KV alignment resolves this by constraining the virtual tokens to lie in a geometrically valid region of the KV representation space before any query-conditioned supervision is applied. %

\subsection{Omitted Experiment Results}\label{app_omitted_results}

We illustrate additional evaluation results across different model scales and sequence lengths to supplement the findings in the main text.

\paragraph{Sensitivity to small numbers of demonstrations.}
To verify on shorter context lengths, Table~\ref{tab_error_k=25} extends the relative output error analysis to the context window of $k=25$ demonstrations. Across two different models (Qwen-1.5B, Qwen-3B) and four downstream tasks, the relative error between the full demonstration prefix and our compressed virtual tokens remains below $0.6\%$. These results confirm that Algorithm~\ref{alg_lcc} consistently preserves the original predictive distribution regardless of the context scale or model architecture.

\paragraph{Results on smaller models.}
To test whether Algorithm~\ref{alg_lcc} maintains its advantage at a smaller model scale, Table~\ref{tab_omitted_error_1.5B} details the computational cost, memory usage, and relative error for the Qwen-1.5B-Instruct model on the SST-2, Poem Sentiment, and Modular Addition datasets. Under the standard setting of $k=50$ demonstrations, Algorithm~\ref{alg_lcc} achieves the lowest relative error among all evaluated methods, including $0.3\%$ on SST-2 and $0.1\%$ on both Poem Sentiment and Addition. Furthermore, Algorithm~\ref{alg_lcc} requires fewer FLOPs than the leading baseline, BSA, while maintaining a comparable memory footprint. This verifies that the efficiency and accuracy advantages of the Algorithm~\ref{alg_lcc} module generalize reliably to smaller model variants.

\paragraph{Coverage of remaining task categories. }
To confirm the extension to other reasoning tasks, Table~\ref{tab_omitted_error_1.5B} and Table~\ref{tab_omitted_baseline_error_1.5B_3B} present the evaluation on the remaining reasoning datasets, Coin Flip and Edge Existence, using both the Qwen-1.5B-Instruct and Qwen-3B-Instruct models with $k=50$ demonstrations. Consistent with earlier findings, Algorithm~\ref{alg_lcc} achieves the lowest relative error among all baselines, recording $0.1\%$ error on both datasets for the 1.5B model, and under $0.9\%$ for the 3B model. This demonstrates that Algorithm~\ref{alg_lcc} preserves strong predictive fidelity across diverse reasoning tasks while maintaining a highly efficient computational profile compared to sliding-window and block-sparse baselines.

\paragraph{Results regarding Algorithm~\ref{alg_aslcc}.}
To separate the effect of the selection algorithm from that of context compression, in Table~\ref{tab_selection_wo_compression}, we find that removing context compression changes the average accuracy from $85.04\%$ to $85.30\%$, corresponding to an increase of $0.26$ percentage points. This result shows that the accuracy gain cannot be explained by the compressed representation acting as a regularizer. The selection algorithm remains effective when candidates are evaluated using the full context.

\begin{table*}[t]
\centering
\resizebox{1.0\textwidth}{!}
{\small
\begin{tabular}{@{}l|c|ccc|ccc|ccc@{}}
\toprule
\multirow{2}{*}{Method} & Window & \multicolumn{3}{c|}{SST-2} & \multicolumn{3}{c|}{Poem Sentiment} & \multicolumn{3}{c}{Modular Addition} \\ 
& Length & Error & FLOPs & Memory & Error & FLOPs & Memory & Error & FLOPs & Memory \\\midrule
\multirow{3}{*}{StreamingLLM}   & $32$ & $11.3\%$ & $9.42e^{12}$ & $6.94$G & $17.5\%$ & $1.36e^{13}$ & $6.61$G & $12.8\%$ & $9.53e^{12}$ & $7.21$G \\
                                & $64$ & $7.6\%$ & $9.43e^{12}$ & $6.94$G & $13.2\%$ & $1.36e^{13}$ & $6.62$G & $11.9\%$ & $9.54e^{12}$ & $7.22$G \\
                                & $128$ & $5.2\%$ & $9.46e^{12}$ & $6.96$G & $8.7\%$ & $1.37e^{13}$ & $6.64$G & $8.4\%$ & $9.56e^{12}$ & $7.24$G \\\midrule
\multirow{3}{*}{LM-Infinite}    & $32$ & $16.2\%$ & $9.36e^{12}$ & $6.93$G & $7.0\%$ & $1.37e^{13}$ & $6.61$G & $5.5\%$ & $9.41e^{12}$ & $7.21$G \\
                                & $64$ & $11.5\%$ & $9.38e^{12}$ & $6.94$G & $3.7\%$ & $1.37e^{13}$ & $6.61$G & $3.5\%$ & $9.44e^{12}$ & $7.21$G \\
                                & $128$ & $11.0\%$ & $9.41e^{12}$ & $6.96$G & $2.9\%$ & $1.37e^{13}$ & $6.64$G & $3.2\%$ & $9.50e^{12}$ & $7.23$G \\\midrule
 BSA & $\sim400$ & $1.1\%$ & $9.52e^{12}$ & $7.15$G & $0.7\%$ & $1.53e^{13}$ & $6.96$G & $0.9\%$ & $8.81e^{12}$ & $7.78$G \\\midrule
 Algorithm~\ref{alg_lcc} & $\mathbf{16}$ & $\mathbf{0.3\%}$ & $\mathbf{6.30e^{12}}$ & $\mathbf{7.19}$G & $\mathbf{0.1\%}$ & $\mathbf{1.02e^{13}}$ & $\mathbf{7.06}$G & $\mathbf{0.1\%}$ & $\mathbf{5.28e^{12}}$ & $\mathbf{8.06}$G \\
\midrule
\multirow{2}{*}{Method} & Window & \multicolumn{3}{c|}{Coin Flip} & \multicolumn{3}{c|}{Edge Existence} & \multicolumn{3}{c}{MMLU} \\ 
& Length & Error & FLOPs & Memory & Error & FLOPs & Memory & Error & FLOPs & Memory \\\midrule
\multirow{3}{*}{StreamingLLM}   & $32$ & $18.1\%$ & $1.30e^{13}$ & $7.83$G & $3.8\%$ & $5.99e^{13}$ & $14.63$G & $3.2\%$ & $3.88e^{13}$ & $10.06$G \\
                                & $64$ & $10.2\%$ & $1.30e^{13}$ & $7.84$G & $2.7\%$ & $6.01e^{13}$ & $14.64$G & $0.8\%$ & $3.88e^{13}$ & $10.08$G  \\
                                & $128$ & $9.8\%$ & $1.31e^{13}$ & $7.85$G & $2.5\%$ & $6.03e^{13}$ & $14.65$G & $0.6\%$ & $3.89e^{13}$ & $10.13$G \\\midrule
\multirow{3}{*}{LM-Infinite}    & $32$ & $28.7\%$ & $1.32e^{13}$ & $7.84$G & $7.5\%$ & $6.43e^{13}$ & $14.63$G & $1.4\%$ & $3.88e^{13}$ & $10.07$G \\
                                & $64$ & $26.8\%$ & $1.32e^{13}$ & $7.84$G & $6.5\%$ & $6.43e^{13}$ & $14.64$G & $0.7\%$ & $3.89e^{13}$ & $10.08$G \\
                                & $128$ & $23.9\%$ & $1.33e^{13}$ & $7.86$G & $4.6\%$ & $6.46e^{13}$ & $14.66$G & $0.5\%$ & $3.89e^{13}$ & $10.13$G \\\midrule
 BSA                            & $\sim400$ & $0.7\%$ & $1.35e^{13}$ & $8.28$G & $0.9\%$ & $6.88e^{13}$ & $15.77$G & $0.5\%$ & $3.95e^{13}$ & $10.36$G \\\midrule
 \textbf{Algorithm~\ref{alg_lcc}} & $\mathbf{16}$ & $\mathbf{0.1\%}$ & $\mathbf{9.02e^{12}}$ & $\mathbf{8.35}$G & $\mathbf{0.1\%}$ & $\mathbf{4.11e^{13}}$ & $\mathbf{14.73}$G & $\mathbf{0.3\%}$ & $\mathbf{2.49e^{13}}$ & $\mathbf{10.81}$G \\
\bottomrule
\end{tabular}}
\caption{Supplementary results on Qwen-1.5B-Instruct. We compare long-context inference methods on SST-2, Poem Sentiment, Modular Addition, Coin Flip, Edge Existence, and MMLU with $k=50$ demonstrations, and report relative output error with respect to full-prefix inference, together with FLOPs and peak memory usage.}\label{tab_omitted_error_1.5B}
\end{table*}

\begin{table*}[t]
\centering
\resizebox{1.0\textwidth}{!}
{\small
\begin{tabular}{@{}l|c|ccc|ccc|ccc@{}}
\toprule
\multirow{2}{*}{Method} & Window & \multicolumn{3}{c|}{Coin Flip} & \multicolumn{3}{c|}{Edge Existence} & \multicolumn{3}{c}{Poem Sentiment} \\ 
& Length & Error & FLOPs & Memory & Error & FLOPs & Memory & Error & FLOPs & Memory \\\midrule
Dense & - & - & $1.49e^{15}$ & $14.58$G & - & $9.77e^{15}$ & $24.57$G & - & $2.53e^{15}$ & $12.70$G \\\midrule
\multirow{2}{*}{StreamingLLM}   & $64$ & $31.5\%$ & $2.63e^{13}$ & $13.59$G & $1.9\%$ & $1.28e^{14}$ & $20.40$G & $9.3\%$ & $2.72e^{13}$ & $12.37$G \\
                                & $128$ & $30.6\%$ & $2.64e^{13}$ & $13.61$G & $1.8\%$ & $1.28e^{14}$ & $20.42$G &  $9.1\%$ & $2.73e^{13}$ & $12.39$G \\\midrule
\multirow{2}{*}{LM-Infinite}    & $64$ & $32.8\%$ & $2.64e^{13}$ & $13.60$G & $1.3\%$ & $1.28e^{14}$ & $20.40$G & $14.2\%$ & $2.72e^{13}$ & $12.38$G \\
                                & $128$ & $30.6\%$ & $2.64e^{13}$ & $13.61$G & $1.2\%$ & $1.28e^{14}$ & $20.43$G & $13.8\%$ & $2.73e^{13}$ & $12.40$G \\\midrule
 ICAE                           & $128$ & $1.4\%$ & $1.36e^{14}$ & $14.27$G & $3.5\%$ & $1.02e^{15}$ & $24.76$G & $3.2\%$ & $5.40e^{13}$ & $12.32$G \\\midrule
 DuoAttention                  & $64$ & $10.5\%$ & $2.70e^{13}$ & $14.02$G & $3.2\%$ & $1.40e^{14}$ & $22.59$G &  $7.6\%$ & $2.75e^{13}$ & $12.50$G \\\midrule
 BSA                            & $\sim400$ & $2.4\%$ & $2.75e^{13}$ & $14.07$G & $1.4\%$ & $1.53e^{14}$ & $24.64$G &  $4.2\%$ & $2.95e^{13}$ & $12.85$G \\\midrule
 \textbf{Algorithm~\ref{alg_lcc}} & $\mathbf{16}$ & $\mathbf{0.7\%}$ & $\mathbf{1.91e^{13}}$ & $\mathbf{14.13}$G & $\mathbf{0.8\%}$ & $\mathbf{9.97e^{13}}$ & $\mathbf{25.02}$G & $\mathbf{1.6\%}$ & $\mathbf{2.37e^{13}}$ & $\mathbf{13.23}$G\\
\bottomrule
\end{tabular}}
\caption{Supplementary results on Qwen-3B-Instruct. We compare long-context inference methods on the Coin Flip, Edge Existence, and Poem Sentiment tasks with $k=50$ demonstrations, and report relative output error with respect to full-prefix inference, together with FLOPs and peak memory usage.}\label{tab_omitted_baseline_error_1.5B_3B}
\end{table*}